\newif\iflong
\longtrue

\documentclass{article}

\usepackage{times}
\usepackage{graphicx} 
\usepackage{subfigure} 

\usepackage{natbib}

\usepackage{algorithm}
\usepackage{algorithmic}

\usepackage{hyperref}

\iflong
\usepackage[accepted]{icml2016} 
\else
\usepackage{fullpage}
\fi

\usepackage{multirow}
\usepackage{amsmath,amsthm,amsfonts,amssymb,mathrsfs}
\usepackage{enumitem} 
\usepackage[skip=5pt]{caption}

\newcommand{\bz}{\boldsymbol{z}}
\newcommand{\bx}{\boldsymbol{x}}

\newcommand{\by}{\boldsymbol{y}}

\newcommand{\be}{\boldsymbol{e}}

\newcommand{\bw}{\boldsymbol{w}}

\newcommand{\bv}{\boldsymbol{v}}
\newcommand{\bzero}{\boldsymbol{0}}

\newcommand{\bdelta}{\boldsymbol{\delta}}

\DeclareMathOperator*{\argmin}{argmin}
\DeclareMathOperator*{\argmax}{argmax}

\newcommand{\field}[1]{\mathbb{#1}}

\newcommand{\R}{\field{R}}

\newcommand{\E}{\field{E}}

\newcommand{\norm}[1]{\left\|{#1}\right\|}

\newcommand{\defeq}{\stackrel{\rm def}{=}}

\newcommand{\scO}{\mathcal{O}}

\newtheorem{lemma}{Lemma}
\newtheorem{theorem}{Theorem}
\newtheorem{cor}{Corollary}

\newcommand{\specialcell}[2][c]{\begin{tabular}[#1]{@{}c@{}}#2\end{tabular}}

\iflong
\icmltitlerunning{Variance-Reduced and Projection-Free Stochastic Optimization}
\else
\title{Variance-Reduced and Projection-Free Stochastic Optimization}
\author{Elad Hazan \\ Princeton University \\ Princeton, NJ 08540 \\ ehazan@cs.princeton.edu
\and
Haipeng Luo \\ Princeton University \\ Princeton, NJ 08540 \\ haipengl@cs.princeton.edu
}
\date{}
\fi

\begin{document} 
\iflong
\twocolumn[
\icmltitle{Variance-Reduced and Projection-Free Stochastic Optimization }

\icmlauthor{Elad Hazan}{ehazan@cs.princeton.edu}
\icmladdress{Princeton University, Princeton, NJ 08540, USA}
\icmlauthor{Haipeng Luo}{haipengl@cs.princeton.edu}
\icmladdress{Princeton University, Princeton, NJ 08540, USA}

\icmlkeywords{variance reduction, Frank-Wolfe algorithm, stochastic optimization, projection-free, multiclass classification}

\vskip 0.3in
]
\else
\maketitle
\fi
\begin{abstract} 
The Frank-Wolfe optimization algorithm has recently regained popularity for machine learning applications due to its projection-free property and its ability to handle structured constraints. 
However, in the stochastic learning setting, it is still relatively understudied compared to the gradient descent counterpart. 
In this work, leveraging a recent variance reduction technique, we propose two stochastic Frank-Wolfe variants which substantially improve previous results in terms of the number of stochastic gradient evaluations needed to achieve $1-\epsilon$ accuracy.
For example, we improve from $\scO(\frac{1}{\epsilon})$ to $\scO(\ln\frac{1}{\epsilon})$ if the objective function is smooth and strongly convex, and from $\scO(\frac{1}{\epsilon^2})$ to $\scO(\frac{1}{\epsilon^{1.5}})$ if the objective function is smooth and Lipschitz. 
The theoretical improvement is also observed in experiments on real-world datasets for a multiclass classification application.
\end{abstract} 

\section{Introduction}

We consider the following optimization problem
\[ \min_{\bw \in \Omega} f(\bw) = \min_{\bw \in \Omega} \frac{1}{n} \sum_{i=1}^n f_i(\bw) \]
which is an extremely common objective in machine learning.
We are interested in the case where 1) $n$, usually corresponding to the number of training examples, is very large and therefore stochastic optimization is much more efficient; and 2) the domain $\Omega$ admits fast linear optimization, while projecting onto it is much slower, necessitating projection-free optimization algorithms.
Examples of such problem include multiclass classification, multitask learning, recommendation  systems, matrix learning and many more (see for example~\citep{HazanKa12,HazanKS12,Jaggi13,DudikHaMa12, ZhangScYu12, HarchaouiJuNe15}).

The Frank-Wolfe algorithm~\citep{FrankWo56} (also known as {\it conditional gradient}) and it variants are natural candidates for solving these problems, due to its projection-free property and its ability to handle structured constraints. 
However, despite gaining more popularity recently, its applicability and efficiency in the stochastic learning setting,
where computing stochastic gradients is much faster than computing exact gradients, is still relatively understudied compared to variants of projected gradient descent methods.

In this work, we thus try to answer the following question: {\it what running time can a projection-free algorithm achieve in terms of the number of stochastic gradient evaluations and the number of linear optimizations needed to achieve a certain accuracy?}
Utilizing Nesterov's acceleration technique~\citep{Nesterov83} and the recent variance reduction idea~\citep{JohnsonZh13, MahdaviZhJi13}, we propose two new algorithms that are substantially faster than previous work.
Specifically, to achieve $1-\epsilon$ accuracy, while the number of linear optimization is the same as previous work, the improvement of the number of stochastic gradient evaluations is summarized in Table~\ref{tab:short_comparisons}:

\begin{table}[h!]
\begin{center}
\begin{tabular}{|c|c|c|}
\hline
~ &  previous work  & \textbf{this work}  \\
\hline
Smooth & $\scO(\frac{1}{\epsilon^2})$  & $\scO(\frac{1}{\epsilon^{1.5}})$ \\
\hline
\specialcell{\footnotesize Smooth and \\ \footnotesize Strongly Convex} & $\scO(\frac{1}{\epsilon})$  & $\scO(\ln \frac{1}{\epsilon})$ \\
\hline
\end{tabular}
\end{center}
\caption{Comparisons of number of stochastic gradients}
\label{tab:short_comparisons}
\end{table}

The extra overhead of our algorithms is computing at most $\scO(\ln \frac{1}{\epsilon})$ exact gradients, which is computationally insignificant compared to the other operations. 
A  more detailed comparisons to previous work is included in Table~\ref{tab:comparisons}, which will be further explained in Section~\ref{sec:prelim}.

While the idea of our algorithms is quite straightforward, we emphasize that our analysis is non-trivial, especially for the second algorithm where the convergence of a sequence of auxiliary points in Nesterov's algorithm needs to be shown.

To support our theoretical results, we also conducted experiments on three large real-word datasets for a multiclass classification application. These experiments show significant improvement over both previous projection-free algorithms and algorithms such as projected stochastic gradient descent and its variance-reduced version.

The rest of the paper is organized as follows: 
Section~\ref{sec:prelim} setups the problem more formally and discusses related work. 
Our two new algorithms are presented and analyzed in Section~\ref{sec:SVRF} and~\ref{sec:STORC}, followed by experiment details in Section~\ref{sec:experiments}.

\section{Preliminary and Related Work}\label{sec:prelim}
We assume each function $f_i$ is convex and $L$-smooth in $\R^{d}$ so that for any $\bw, \bv \in \R^{d}$,\footnote{%
We thank Sebastian Pokutta and G\'abor Braun for pointing
out that $f_i$ needs to be defined over $\R^d$, rather than only over $\Omega$, in order for property~\eqref{eq:smoothness1} to hold.}
\begin{equation*}
\begin{split}
 \nabla f_i(\bv)^\top & (\bw - \bv)  \leq f_i(\bw) - f_i(\bv)  \\
& \leq \nabla f_i(\bv)^\top (\bw - \bv) + \frac{L}{2}\norm{\bw - \bv}^2. 
\end{split}
\end{equation*}
We will use two more important properties of smoothness. The first one is
\begin{equation}\label{eq:smoothness1}
\begin{split}
 \| \nabla f_i(\bw) &- \nabla f_i(\bv) \|^2 \leq \\
 &2L(f_i(\bw) - f_i(\bv) - \nabla f_i(\bv)^\top (\bw - \bv))
\end{split}
\end{equation}
(proven in Appendix~\ref{app:smoothness} for completeness),
and the second one is
\begin{equation}\label{eq:smoothness2}
\begin{split}
 f_i(& \lambda \bw +  (1 - \lambda) \bv)  \geq \\
& \lambda f_i(\bw) + (1-\lambda) f_i(\bv) - \frac{L}{2} \lambda(1 - \lambda) \norm{\bw - \bv}^2
\end{split}
\end{equation}
for any $\bw, \bv \in \Omega$ and $\lambda \in [0,1]$. 
Notice that $f = \frac{1}{n}\sum_{i=1}^n f_i$ is also $L$-smooth since smoothness is preserved under convex combinations.

For some cases, we also assume each $f_i$ is $G$-Lipschitz: $\norm{\nabla f_i(\bw)} \leq G$ for any $\bw \in \Omega$,
and $f$ (although not necessarily each $f_i$) is $\alpha$-strongly convex, that is,
\[ f(\bw) - f(\bv) \leq \nabla f(\bw)^\top (\bw - \bv) - \frac{\alpha}{2}\norm{\bw - \bv}^2 \]
for any $\bw, \bv \in \Omega$. 
As usual, $\mu = \frac{L}{\alpha}$ is called the condition number of $f$. 

We assume the domain $\Omega \subset \R^{d}$ is a compact convex set with diameter $D$. 
We are interested in the case where linear optimization on $\Omega$, formally $\argmin_{\bv \in \Omega} \bw^\top \bv$ for any $\bw \in \R^{d}$, is much faster than projection onto $\Omega$, formally $\argmin_{\bv \in \Omega} \norm{\bw - \bv}^2$.
Examples of such domains include the set of all bounded trace norm matrices, the convex hull of all rotation matrices, flow polytope and many more (see for instance~\citep{HazanKa12}).

\subsection{Example Application: Multiclass Classification}\label{subsec:multiclass}

Consider a multiclass classification problem where a set of training examples $(\be_i, y_i)_{i = 1,\ldots, n}$ is given beforehand. 
Here $\be_i \in \R^{m}$ is a feature vector and $y_i \in \{1,\ldots,h\}$ is the label. 
Our goal is to find an accurate linear predictor, a matrix $\bw = [\bw_1^\top; \ldots, \bw_h^\top] \in \R^{h\times m}$ that predicts $\argmax_\ell \bw_\ell^\top \be$ for any example $\be$. Note that here the dimensionality $d$ is $hm$.

Previous work~\citep{DudikHaMa12, ZhangScYu12} found that finding $\bw$ by minimizing a regularized multivariate logistic loss gives a very accurate predictor in general. 
Specifically, the objective can be written in our notation with
\[ f_i(\bw) = \log\bigg( 1 + \sum_{\ell\neq y_i} \exp(\bw_\ell^\top \be_i - \bw_{y_i}^\top \be_i)\bigg) \]
and $ \Omega = \{\bw \in \R^{h\times m} : \|\bw\|_* \leq \tau \} $ where $\norm{\cdot}_*$ denotes the matrix trace norm.
In this case, projecting onto $\Omega$ is equivalent to performing an SVD, which takes $\scO(hm\min\{h, m\})$ time,
while linear optimization on $\Omega$ amounts to finding the top singular vector, which can be done in time linear to the number of non-zeros in the corresponding $h$ by $m$ matrix, and is thus much faster.
One can also verify that each $f_i$ is smooth.
The number of examples $n$ can be prohibitively large for non-stochastic methods (for instance, tens of millions for the ImageNet dataset~\citep{DengDoSoLiLiLi09}),
which makes stochastic optimization  necessary.

\subsection{Detailed Efficiency Comparisons}\label{subsec:comparisons}

\begin{table*}[t]
\vskip 0.15in
\centering
\begin{tabular}{|c|c|c|c|c|}
\hline
Algorithm & Extra Conditions & \#Exact Gradients & \#Stochastic Gradients & \#Linear Optimizations \\
\hline
Frank-Wolfe & &   $\scO(\frac{LD^2}{\epsilon})$  & 0 & $\scO(\frac{LD^2}{\epsilon})$ \\
\hline
\citep{GarberHa13} & \specialcell{{\small $\alpha$-strongly convex} \\ {\footnotesize $\Omega$ is polytope}}  & $\scO(d\mu \rho \ln \frac{LD^2}{\epsilon})$  & 0 & $\scO(d\mu \rho \ln \frac{LD^2}{\epsilon})$\\
\hline
SFW & $G$-Lipschitz & 0 & $\scO(\frac{G^2 LD^4}{\epsilon^3})$  & $\scO(\frac{LD^2}{\epsilon})$ \\
\hline
\multirow{2}{*}[-4pt]{\specialcell{Online-FW \\ \citep{HazanKa12}}} & $G$-Lipschitz & 0 & $\scO(\frac{d^2(LD^2+GD)^4}{\epsilon^4})$  & $\scO(\frac{d(LD^2+GD)^2}{\epsilon^2})$ \\
\cline{2-5}
& \specialcell{{\small $G$-Lipschitz} \\ {\scriptsize ($L=\infty$ allowed)}}& 0 & $\scO(\frac{G^4D^4}{\epsilon^4})$  & $\scO(\frac{G^4D^4}{\epsilon^4})$ \\
\hline
\multirow{2}{*}[-4pt]{\specialcell{SCGS \\ \citep{LanZh14}}} & $G$-Lipschitz & 0 & $\scO(\frac{G^2D^2}{\epsilon^2})$  & $\scO(\frac{LD^2}{\epsilon})$ \\
\cline{2-5}
& \specialcell{{\small $G$-Lipschitz} \\ {\small $\alpha$-strongly convex}} & 0 & $\scO(\frac{G^2}{\alpha\epsilon})$  & $\scO(\frac{LD^2}{\epsilon})$ \\
\hline
\textbf{SVRF (this work)}  &  & $\scO(\ln\frac{LD^2}{\epsilon})$ & $\scO(\frac{L^2D^4}{\epsilon^2})$ & $\scO(\frac{LD^2}{\epsilon})$ \\
\hline
\multirow{3}{*}{\textbf{STORC (this work)}}  & $G$-Lipschitz & $\scO(\ln\frac{LD^2}{\epsilon})$ & $\scO(\frac{\sqrt{L}D^2G}{\epsilon^{1.5}})$ & $\scO(\frac{LD^2}{\epsilon})$ \\
\cline{2-5}
& $\nabla f(\bw^*) = \bzero$ & $\scO(\ln\frac{LD^2}{\epsilon})$ & $\scO(\frac{LD^2}{\epsilon})$ & $\scO(\frac{LD^2}{\epsilon})$ \\
\cline{2-5}
 & $\alpha$-strongly convex & $\scO(\ln\frac{LD^2}{\epsilon})$ & $\scO(\mu^2 \ln \frac{LD^2}{\epsilon})$ & $\scO(\frac{LD^2}{\epsilon})$ \\
\hline
\end{tabular}
\caption{\label{tab:comparisons}  Comparisons of different Frank-Wolfe variants (see Section~\ref{subsec:comparisons} for further explanations). }
\end{table*}

We call $\nabla f_i(\bw)$ a {\it stochastic gradient} for $f$ at some $\bw$, where $i$ is picked from $\{1,\ldots,n\}$ uniformly at random.
Note that a stochastic gradient $\nabla f_i(\bw)$ is an unbiased estimator of the {\it exact gradient} $\nabla f(\bw)$.
The efficiency of a projection-free algorithm is measured by how many numbers of exact gradient evaluations, 
stochastic gradient evaluations and linear optimizations respectively are needed to achieve $1-\epsilon$ accuracy, that is, to output a point $\bw \in \Omega$  such that $\E[f(\bw) - f(\bw^*)] \leq \epsilon$ where $\bw^* \in \argmin_{\bw \in \Omega} f(\bw)$ is any optimum.

In Table~\ref{tab:comparisons}, we summarize the efficiency (and extra assumptions needed beside convexity and smoothness%
\footnote{In general, condition ``$G$-Lipschitz'' in Table~\ref{tab:comparisons} means each $f_i$ is $G$-Lipschitz, except for our STORC algorithm which only requires $f$ being $G$-Lipschitz.}) of existing algorithms in the literature as well as  the two new algorithms we propose.
Below we briefly explain these results from top to bottom.

The standard Frank-Wolfe algorithm:
\begin{equation}\label{eq:FW}
\begin{split}
\bv_k &= \argmin_{\bv \in \Omega} \nabla f(\bw_{k-1})^\top \bv \\
\bw_k &= (1 - \gamma_k) \bw_{k-1} + \gamma_k \bv_k
\end{split}
\end{equation}
for some appropriate chosen $\gamma_k$ requires $\scO(\frac{1}{\epsilon})$ iteration without additional conditions~\citep{FrankWo56,Jaggi13}. In a recent paper, \citet{GarberHa13} give a variant that requires $\scO( d \mu \rho \ln\frac{1}{\epsilon})$ iterations when $f$ is strongly convex and smooth, and $\Omega$ is a polytope\footnote{See also recent follow up work \cite{LacosteJulienJa15}.}. Although the dependence on $\epsilon$ is much better, the geometric constant  $\rho$ depends on the polyhedral set and can be very large. Moreover, each iteration of the algorithm requires further computation besides the linear optimization step.  

The most obvious way to obtain a stochastic Frank-Wolfe variant is to replace $\nabla f(\bw_{k-1})$ by some $\nabla f_i(\bw_{k-1})$, or more generally the average of some number of iid samples of $\nabla f_i(\bw_{k-1})$ (mini-batch approach).
We call this method SFW and include its analysis in Appendix~\ref{app:SFW} since we do not find it explicitly analyzed before.  
SFW needs $\scO(\frac{1}{\epsilon^3})$ stochastic gradients and $\scO(\frac{1}{\epsilon})$ linear optimization steps to reach an $\epsilon$-approximate optimum.

The work by~\citet{HazanKa12} focuses on a online learning setting. 
One can extract two results from this work for the setting studied here.\footnote{The first result comes from the setting where the online loss functions are stochastic, and the second one comes from a completely online setting with the standard online-to-batch conversion.}
In any case, the result is worse than SFW for both the number of stochastic gradients and the number of linear optimizations.

Stochastic Condition Gradient Sliding (SCGS), recently proposed by~\cite{LanZh14}, 
uses Nesterov's acceleration technique to speed up Frank-Wolfe.
Without strong convexity, SCGS needs $\scO(\frac{1}{\epsilon^2})$ stochastic gradients, improving SFW.
With strong convexity, this number can even be improved to $\scO(\frac{1}{\epsilon})$.
In both cases, the number of linear optimization steps is $\scO(\frac{1}{\epsilon})$.
 
The key idea of our algorithms is to combine the variance reduction technique proposed in~\citep{JohnsonZh13, MahdaviZhJi13} with some of the above-mentioned algorithms.
For example,  our algorithm SVRF combines this technique with SFW, also improving the number of stochastic gradients from $\scO(\frac{1}{\epsilon^3})$ to  $\scO(\frac{1}{\epsilon^2})$, but without any extra conditions (such as Lipschitzness required for SCGS).
More importantly, despite having seemingly same convergence rate, SVRF substantially outperforms SCGS empirically (see Section~\ref{sec:experiments}).

On the other hand, our second algorithm STORC combines variance reduction with SCGS, providing even further improvements.
Specifically, the number of stochastic gradients is improved to: $\scO(\frac{1}{\epsilon^{1.5}})$ when $f$ is Lipschitz;
$\scO(\frac{1}{\epsilon})$ when $\nabla f(\bw^*) = \bzero$;
and finally $\scO(\ln\frac{1}{\epsilon})$ when $f$ is strongly convex.
Note that the condition $\nabla f(\bw^*) = \bzero$ essentially means that $\bw^*$ is in the interior of $\Omega$, 
but it is still an interesting case when the optimum is not unique and doing unconstraint optimization would not necessary return a point in $\Omega$.

Both of our algorithms require $\scO(\frac{1}{\epsilon})$ linear optimization steps as previous work, and overall require computing $\scO(\ln\frac{LD^2}{\epsilon})$ exact gradients.
However, we emphasize that this extra overhead is much more affordable compared to non-stochastic Frank-Wolfe (that is, computing exact gradients every iteration) since it does not have any polynomial dependence on parameters such as $d$, $L$ or $\mu$.

\subsection{Variance-Reduced Stochastic Gradients}
Originally proposed in~\citep{JohnsonZh13} and independently in~\citep{MahdaviZhJi13}, the idea of variance-reduced stochastic gradients is proven to be highly useful and has been extended to various different algorithms (such as~\citep{FrostigGeKaSi15, MoritzNiJo15}).

A variance-reduced stochastic gradient at some point $\bw \in \Omega$ with some {\it snapshot} $\bw_0 \in \Omega$ is defined as 
\[ \tilde\nabla f(\bw; \bw_0) = \nabla f_i(\bw) - (\nabla f_i(\bw_0) - \nabla f(\bw_0)), \]
where $i$ is again picked from $\{1,\ldots,n\}$ uniformly at random.
The snapshot $\bw_0$ is usually a decision point from some previous iteration of the algorithm and its exact gradient $\nabla f(\bw_0)$ has been pre-computed before, so that computing $\tilde\nabla f(\bw; \bw_0)$ only requires two standard stochastic gradient evaluations: $\nabla f_i(\bw)$ and $\nabla f_i(\bw_0)$. 

A variance-reduced stochastic gradient is clearly also unbiased, that is, $\E[\tilde\nabla f(\bw; \bw_0)] = \nabla f(\bw)$.
More importantly, the term $\nabla f_i(\bw_0) - \nabla f(\bw_0)$ serves as a correction term to reduce the variance of the stochastic gradient.
Formally, one can prove the following

\begin{lemma}\label{lem:variance}
For any $\bw, \bw_0 \in \Omega$, we have
\begin{equation*}
\begin{split}
&\E[\| \tilde\nabla f(\bw; \bw_0)  - \nabla f(\bw) \|^2 ] \\
& \leq 6L(2\E[f(\bw) - f(\bw^*)] + \E[f(\bw_0) - f(\bw^*)]).
\end{split}
\end{equation*}
\end{lemma}

In words, the variance of the variance-reduced stochastic gradient is bounded by how close the current point and the snapshot are to the optimum. 
The original work proves a bound on $\E[\| \tilde\nabla f(\bw; \bw_0)\|^2]$ under the assumption $\nabla f(\bw^*) = \bzero$,
which we do not require here. 
However, the main idea of the proof is similar and we defer it to Section~\ref{sec:proofs}.

\section{Stochastic Variance-Reduced Frank-Wolfe}\label{sec:SVRF}

With the previous discussion, our first algorithm is pretty straightforward:
compared to the standard Frank-Wolfe, we simply replace the exact gradient with the average of a mini-batch of variance-reduced stochastic gradients, and take snapshots every once in a while. 
We call this algorithm Stochastic Variance-Reduced Frank-Wolfe (SVRF), whose pseudocode is presented in Alg~\ref{alg:SVRF}.
The convergence rate of this algorithm is shown in the following theorem.

\begin{algorithm}[t]
\caption{Stochastic Variance-Reduced Frank-Wolfe (SVRF)}
\label{alg:SVRF}
\begin{algorithmic}[1]
\STATE {\bfseries Input:} Objective function $f = \frac{1}{n}\sum_{i=1}^n f_i$. 
\STATE {\bfseries Input:} Parameters $\gamma_k$, $m_k$ and $N_k$.
\STATE {\bfseries Initialize:} $\bw_0 = \argmin_{\bw \in \Omega} \nabla f(\bx)^\top \bw$ for some arbitrary $\bx \in \Omega$.
\FOR{$t=1, 2, \ldots, T $}
    \STATE Take snapshot: $\bx_0 = \bw_{t-1}$ and compute $\nabla f(\bx_0)$.
    \FOR{$k=1$ {\bfseries to} $N_t$} 
        \STATE Compute $\tilde\nabla_k$, the average of $m_k$ iid samples of $\tilde\nabla f(\bx_{k-1}, \bx_0)$.
        \STATE Compute $\bv_k = \argmin_{\bv \in \Omega} \tilde\nabla_k^\top \bv$.
        \STATE Compute $\bx_k = (1 - \gamma_k) \bx_{k-1} + \gamma_k \bv_k$.
    \ENDFOR  
    \STATE Set $\bw_t = \bx_{N_t}$.
\ENDFOR 
\end{algorithmic}
\end{algorithm}

\begin{theorem}\label{thm:SVRF}
With the following parameters,
\[ \gamma_k = \frac{2}{k+1}, \;  m_k = 96(k+1), \; N_t =  2^{t+3} - 2, \]
Algorithm~\ref{alg:SVRF} ensures 
$ \E[f(\bw_t) - f(\bw^*)] \leq \frac{LD^2}{2^{t+1}}$ for any $t$.
\end{theorem}

Before proving this theorem, we first show a direct implication of this convergence result.

\begin{cor}
To achieve $1-\epsilon$ accuracy, Algorithm~\ref{alg:SVRF} requires $\scO(\ln\frac{LD^2}{\epsilon})$ exact gradient evaluations, $\scO(\frac{L^2D^4}{\epsilon^2})$ stochastic gradient evaluations and $\scO(\frac{LD^2}{\epsilon})$ linear optimizations.
\end{cor}
\begin{proof}
According to the algorithm and the choice of parameters, it is clear that these three numbers are $T+1$, $\sum_{t=1}^T\sum_{k=1}^{N_t} m_k = \scO(4^T)$ and $\sum_{t=1}^T N_t = \scO(2^T)$ respectively.
Theorem~\ref{thm:SVRF} implies that $T$ should be of order $\Theta(\log_2\frac{LD^2}{\epsilon})$. 
Plugging in all parameters concludes the proof.
\end{proof}

To prove Theorem~\ref{thm:SVRF}, we first consider a fixed iteration $t$ and prove the following lemma:

\begin{lemma}\label{lem:SVRF}
For any $k$,  we have 
\[ \E[f(\bx_k) - f(\bw^*)] \leq \frac{4LD^2}{k+2} \]
if $ \E[\| \tilde\nabla_s - \nabla f(\bx_{s-1})\|^2 ] \leq \frac{L^2D^2}{(s+1)^2}$ for all $s \leq  k$.
\end{lemma}

We defer the proof of this lemma to Section~\ref{sec:proofs} for coherence.
With the help of Lemma~\ref{lem:SVRF}, we are now ready to prove the main convergence result.

\begin{proof}[Proof of Theorem~\ref{thm:SVRF}]
We prove by induction. For $t=0$, by smoothness, the optimality of $\bw_0$ and convexity, we have
\begin{align*}
f(\bw_0) &\leq f(\bx) + \nabla f(\bx)^\top (\bw_0 - \bx) + \frac{L}{2}\|\bw_0 - \bx\|^2 \\
&\leq f(\bx) + \nabla f(\bx)^\top (\bw^* - \bx) + \frac{LD^2}{2}  \\
&\leq f(\bw^*) + \frac{LD^2}{2}.  
\end{align*}
Now assuming $\E[f(\bw_{t-1}) - f(\bw^*)] \leq \frac{LD^2}{2^{t}}$, we consider iteration $t$ of the algorithm and use another induction to show
$\E[f(\bx_k) - f(\bw^*)] \leq \frac{4LD^2}{k+2}$ for any $k \leq N_t$. 
The base case is trivial since $\bx_0 = \bw_{t-1}$. 
Suppose $\E[f(\bx_{s-1}) - f(\bw^*)]  \leq \frac{4LD^2}{s+1}$ for any $s \leq k$. 
Now because $\tilde\nabla_s$ is the average of $m_s$ iid samples of $\tilde\nabla f(\bx_{s-1}; \bx_0)$,
its variance is reduced by a factor of $m_s$. That is, with Lemma~\ref{lem:variance} we have
\begin{align*}
&\E[\| \tilde\nabla_s - \nabla f(\bx_{s-1})\|^2 ]  \\
\leq\;& \frac{6L}{m_s}(2\E[f(\bx_{s-1}) - f(\bw^*)] + \E[f(\bx_0) - f(\bw^*)] ) \\
\leq\;& \frac{6L}{m_s} \left( \frac{8LD^2}{s+1} + \frac{LD^2}{2^t} \right)  \\
\leq\;& \frac{6L}{m_s} \left( \frac{8LD^2}{s+1} + \frac{8LD^2}{s+1} \right) = \frac{L^2 D^2}{(s+1)^2},
\end{align*}
where the last inequality is by the fact $s \leq N_t = 2^{t+3} - 2$ and the last equality is by plugging the choice of $m_s$.
Therefore the condition of Lemma~\ref{lem:SVRF} is satisfied and the induction is completed.
Finally with the choice of $N_t$ we thus prove $\E[f(\bw_t) - f(\bw^*)] = \E[f(\bx_{N_t}) - f(\bw^*)] \leq \frac{4LD^2}{N_t+2} = \frac{LD^2}{2^{t+1}}$.
\end{proof}

We remark that in Alg~\ref{alg:SVRF}, we essentially restart the algorithm (that is, reseting $k$ to 1) after taking a new snapshot.
However, another option is to continue increasing $k$ and never reset it.
Although one can show that this only leads to constant speed up for the convergence, it provides more stable update and is thus what we implement in experiments.

\section{Stochastic Variance-Reduced Conditional Gradient Sliding}\label{sec:STORC}

Our second algorithm applies variance reduction to the SCGS algorithm~\citep{LanZh14}.
Again, the key difference is that we replace the stochastic gradients with the average of a mini-batch of variance-reduced stochastic gradients, and take snapshots every once in a while. 
See pseudocode in Alg~\ref{alg:STORC} for details.

The algorithm makes use of two auxiliary sequences $\bx_k$ and $\bz_k$ (Line~8 and~12), which is standard for Nesterov's algorithm. $\bx_k$ is obtained by approximately solving a square norm regularized linear optimization so that it is close to $\bx_{k-1}$ (Line~11). 
Note that this step does not require computing any extra gradients of $f$ or $f_i$, and is done by performing the standard Frank-Wolfe algorithm (Eq.~\eqref{eq:FW}) until the duality gap is at most a certain value $\eta_{t,k}$.
The duality gap is a certificate of approximate optimality (see~\citep{Jaggi13}), and is a side product of the linear optimization performed at each step, requiring no extra cost.

Also note that the stochastic gradients are computed at the sequence $\bz_k$ instead of $\by_k$, which is also standard in Nesterov's algorithm. 
However, according to Lemma~\ref{lem:variance}, we thus need to show the convergence rate of the auxiliary sequence $\bz_k$, which appears to be rarely studied previously to the best our knowledge. 
This is one of the key steps in our analysis. 

The main convergence result of STORC is the following:

\begin{algorithm}[t]
\caption{STOchastic variance-Reduced Conditional gradient sliding (STORC)}
\label{alg:STORC}
\begin{algorithmic}[1]
\STATE {\bfseries Input:} Objective function $f = \frac{1}{n}\sum_{i=1}^n f_i$.
\STATE {\bfseries Input:} Parameters $\gamma_k$, $\beta_k$, $\eta_{t,k}$, $m_{t,k}$ and $N_t$.
\STATE {\bfseries Initialize:} $\bw_0 = \argmin_{\bw \in \Omega} \nabla f(\bx)^\top \bw$ for some arbitrary $\bx \in \Omega$.
\FOR{$t = 1, 2, \ldots$}
    \STATE Take snapshot: $\by_0 = \bw_{t-1}$ and compute $\nabla f(\by_0)$.
    \STATE Initialize $\bx_0 =  \by_0$.
    \FOR{$k=1$ {\bfseries to} $N_t$} 
        \STATE Compute $\bz_k = (1 - \gamma_k) \by_{k-1} + \gamma_k \bx_{k-1}$.
        \STATE Compute $\tilde\nabla_k$, the average of $m_{t,k}$ iid samples of $\tilde\nabla f(\bz_k; \by_0)$.
        \STATE Let $g(\bx) = \frac{\beta_k}{2} \norm{\bx - \bx_{k-1}}^2 + \tilde\nabla_k^\top\bx$.
        \STATE Compute $\bx_k$, the output of using standard Frank-Wolfe to solve $\min_{\bx \in \Omega} g(\bx)$ until the duality gap is at most $\eta_{t,k}$, that is, 
        \vspace{-3pt}\begin{equation}\label{eq:gap}
        \max_{\bx \in \Omega} \nabla g(\bx_k)^\top (\bx_k - \bx)  \leq \eta_{t,k} ~. 
        \end{equation} \label{line:WolfeGap} \vspace{-10pt}
        \STATE Compute $\by_k = (1 - \gamma_k) \by_{k-1} + \gamma_k \bx_k$.
    \ENDFOR 
    \STATE Set $\bw_t = \by_{N_t}$.
\ENDFOR
\end{algorithmic}
\end{algorithm}

\begin{theorem}\label{thm:STORC}
With the following parameters (where $D_t$ is defined later below):
\[ \gamma_k = \frac{2}{k+1}, \; \beta_k = \frac{3L}{k}, \; \eta_{t,k} = \frac{2LD_t^2}{N_t k}, \]
Algorithm~\ref{alg:STORC} ensures $\E[f(\bw_t) - f(\bw^*)] \leq \frac{LD^2}{2^{t+1}}$ for any $t$ if any of the following three cases holds:
\begin{enumerate}[label=(\alph*)]
\item $\nabla f(\bw^*) = \bzero$ and $D_t = D, N_t = \lceil 2^{\frac{t}{2}+2} \rceil, m_{t, k} = 900N_t$.

\item $f$ is $G$-Lipschitz and $D_t = D, N_t = \lceil 2^{\frac{t}{2}+2} \rceil, m_{t, k} = 700N_t + \frac{24N_t G (k+1)}{LD}$.

\item $f$ is $\alpha$-strongly convex and $D_t^2 = \frac{\mu D^2}{2^{t-1}}, N_t = \lceil \sqrt{32\mu} \rceil, m_{t,k} = 5600N_t \mu$ where $\mu = \frac{L}{\alpha}$. 

\end{enumerate}
\end{theorem}

Again we first give a direct implication of the above result:
\begin{cor}
To achieve $1-\epsilon$ accuracy,
Algorithm~\ref{alg:STORC} requires $\scO(\ln\frac{LD^2}{\epsilon})$ exact gradient evaluations and $\scO(\frac{LD^2}{\epsilon})$ linear optimizations. 
The numbers of stochastic gradient evaluations for Case (a), (b) and (c) are respectively $\scO(\frac{LD^2}{\epsilon})$, $\scO(\frac{LD^2}{\epsilon} + \frac{\sqrt{L}D^2G}{\epsilon^{1.5}})$ and $\scO(\mu^2 \ln \frac{LD^2}{\epsilon})$.
\end{cor}
\begin{proof}
Line~11 requires $\scO(\frac{\beta_k D^2}{\eta_{t,k}})$ iterations of the standard Frank-Wolfe algorithm since $g(\bx)$ is $\beta_k$-smooth (see e.g. \citep[Theorem 2]{Jaggi13}). So the numbers of exact gradient evaluations, stochastic gradient evaluations and linear optimizations are respectively $T+1$,
$ \sum_{t=1}^T \sum_{k=1}^{N_t} m_{t,k}$ and $\scO(\sum_{t=1}^T \sum_{k=1}^{N_t} \frac{\beta_k D^2}{\eta_{t,k}})$.
Theorem~\ref{thm:STORC} implies that $T$ should be of order $\Theta(\log_2\frac{LD^2}{\epsilon})$. 
Plugging in all parameters proves the corollary.
\end{proof}

To prove Theorem~\ref{thm:STORC},  we again first consider a fixed iteration $t$ and use the following lemma,
which is essentially proven in~\citep{LanZh14}. 
We include a distilled proof in Appendix~\ref{app:STORC} for completeness.

\begin{lemma}\label{lem:STORC}
Suppose $\E[\|\by_{0} - \bw^*\|^2] \leq D_t^2$ holds for some positive constant $D_t \leq D$.
Then for any $k$,  we have 
\[ \E[f(\by_k) - f(\bw^*)] \leq \frac{8LD_t^2}{k(k+1)} \] 
if $ \E[\| \tilde\nabla_s - \nabla f(\bz_s)\|^2 ] \leq \frac{L^2 D_t^2}{N_t(s+1)^2}$ for all $s \leq k$.
\end{lemma}

\begin{proof}[Proof of Theorem~\ref{thm:STORC}]
We prove by induction. The base case $t=0$ holds by the exact same argument as in the proof of Theorem~\ref{thm:SVRF}.
Suppose $\E[f(\bw_{t-1}) - f(\bw^*)] \leq \frac{LD^2}{2^{t}}$ and consider iteration $t$. 
Below we use another induction to prove $\E[f(\by_k) - f(\bw^*)] \leq \frac{8LD_t^2}{k(k+1)}$ for any $1 \leq k \leq N_t$, which will concludes the proof since for any of the three cases, we have
$ \E[f(\bw_{t}) - f(\bw^*)] =  \E[f(\by_{N_t}) - f(\bw^*)]$ which is at most $\frac{8LD_t^2}{N_t^2} \leq \frac{LD^2}{2^{t+1}}$.

We first show that the condition $\E[\|\by_{0} - \bw^*\|^2] \leq D_t^2$ holds.
This is trivial for Cases (a) and (b) when $D_t = D$. 
For Case (c), by strong convexity and the inductive assumption, we have
$\E[\|\by_{0} - \bw^*\|^2] \leq \frac{2}{\alpha} \E[f(\by_0) - f(\bw^*)] \leq \frac{LD^2}{\alpha 2^{t-1}} = D_t^2$.

Next note that Lemma~\ref{lem:variance} implies that $\E[\| \tilde\nabla_s - \nabla f(\bz_s)\|^2 ]$ is at most 
$\frac{6L}{m_{t,s}}(2\E[f(\bz_{s}) - f(\bw^*)] + \E[f(\by_0) - f(\bw^*)])$.
So the key is to bound $\E[f(\bz_{s}) - f(\bw^*)]$.
With $\bz_1 = \by_0$ one can verify that
$ \E[\| \tilde\nabla_1 - \nabla f(\bz_1)\|^2 ]$ is at most $\frac{18L}{m_{t,1}} \E[f(\by_{0}) - f(\bw^*)] \leq \frac{18L^2D^2}{m_{t,1} 2^t}
\leq \frac{L^2 D_t^2}{4N_t }$ for all three cases, and thus $\E[f(\by_s) - f(\bw^*)] \leq \frac{8LD_t^2}{s(s+1)}$ holds for $s=1$ by Lemma~\ref{lem:STORC}.
Now suppose it holds for any $s < k$, below we discuss the three cases separately to show that it also holds for $s = k$.
  
\paragraph{Case (a).}
By smoothness, the condition $\nabla f(\bw^*) = 0$, the construction of $\bz_s$, and Cauchy-Schwarz inequality, we have for any $1 < s \leq k$,
\begin{align*}
 &f(\bz_s) \leq f(\by_{s-1}) + (\nabla f(\by_{s-1}) - \nabla f(\bw^*))^\top(\bz_s - \by_{s-1}) \\
 & \quad\quad + \frac{L}{2} \|\bz_s - \by_{s-1}\|^2  \\
&=  f(\by_{s-1}) + \gamma_s (\nabla f(\by_{s-1}) - \nabla f(\bw^*))^\top(\bx_{s-1} - \by_{s-1}) \\
& \quad\quad + \frac{L\gamma_s^2}{2} \|\bx_{s-1} - \by_{s-1}\|^2  \\
& \leq  f(\by_{s-1}) + \gamma_s D \|\nabla f(\by_{s-1}) - \nabla f(\bw^*)\| + \frac{LD^2\gamma_s^2}{2}. \\
\end{align*}

Property~\eqref{eq:smoothness1} and the optimality of $\bw^*$ implies: 
\begin{align*}
&\|\nabla f(\by_{s-1}) - \nabla f(\bw^*)\|^2  \\
\leq\;& 2L(f(\by_{s-1}) - f(\bw^*) - \nabla f(\bw^*)^\top(\by_{s-1} - \bw^*))  \\
\leq\;& 2L(f(\by_{s-1}) - f(\bw^*) ).
\end{align*}
So subtracting $f(\bw^*)$ and taking expectation on both sides, and applying Jensen's inequality and the inductive assumption, we have
\begin{align*}
&\E[f(\bz_{s}) - f(\bw^*)]  \\
\leq\;& \E[f(\by_{s-1}) - f(\bw^*)] + \gamma_s D \sqrt{2L \E[f(\by_{s-1}) - f(\bw^*)] } \\
& \quad\quad + \frac{2LD^2}{(s+1)^2} \\
\leq\;& \frac{8LD^2}{(s-1)s} + \frac{8LD^2}{(s+1)\sqrt{(s-1)s}} + \frac{2LD^2}{(s+1)^2} < \frac{55LD^2}{(s+1)^2}.
\end{align*}
On the other hand, we have $\E[f(\by_0) - f(\bw^*)]  \leq \frac{LD^2}{2^t} \leq \frac{16LD^2}{(N_t - 1)^2} < \frac{40LD^2}{(N_t+1)^2} \leq \frac{40LD^2}{(s+1)^2}$.
So $\E[\| \tilde\nabla_s - \nabla f(\bz_s)\|^2$ is at most $\frac{900L^2D^2}{m_{t,s} (s+1)^2}$,
and the choice of $m_{t,s}$ ensures that this bound is at most $\frac{L^2 D^2}{N_t(s+1)^2}$, satisfying the condition of Lemma~\ref{lem:STORC} and thus completing the induction.

\paragraph{Case (b).}
With the $G$-Lipschitz condition we proceed similarly and bound $f(\bz_s)$ by
\begin{align*}
 &f(\by_{s-1}) + \nabla f(\by_{s-1})^\top(\bz_s - \by_{s-1}) + \frac{L}{2} \|\bz_s - \by_{s-1}\|^2 \\
=\;&  f(\by_{s-1}) + \gamma_s \nabla f(\by_{s-1})^\top(\bx_{s-1} - \by_{s-1}) + \frac{LD^2\gamma_s^2}{2}  \\
\leq\;&  f(\by_{s-1}) + \gamma_s GD + \frac{LD^2\gamma_s^2}{2}. 
\end{align*}
So using bounds derived previously and the choice of $m_{t,s}$, we bound $\E[\| \tilde\nabla_s - \nabla f(\bz_s)\|^2$ as follows:
\begin{align*}
\frac{6L}{m_{t,s}}\left(\frac{16LD^2}{(s-1)s} + \frac{4GD}{s+1} + \frac{4LD^2}{(s+1)^2} + \frac{40LD^2}{(s+1)^2}\right)  \\
< \frac{6L}{m_{t,s}}\left(\frac{4GD}{s+1} + \frac{116LD^2}{(s+1)^2}\right) < \frac{L^2 D^2}{N_t(s+1)^2}, 
\end{align*}
again completing the induction.

\paragraph{Case (c).}
Using the definition of $\bz_s$ and $\by_s$ and direct calcalution, one can remove the dependence of $\bx_s$ and verify 
\[ \by_{s-1} = \frac{s+1}{2s-1} \bz_s + \frac{s-2}{2s-1}\by_{s-2} \]
for any $s \geq 2$.
Now we apply Property~\eqref{eq:smoothness2} with $\lambda = \frac{s+1}{2s-1}$:
\begin{align*}
&f(\by_{s-1}) \geq \frac{s+1}{2s-1} f(\bz_s) + \frac{s-2}{2s-1}f(\by_{s-2}) \\
& \quad\quad\quad - \frac{L}{2} \frac{(s+1)(s-2)}{(2s-1)^2} \|\bz_s - \by_{s-2} \|^2 \\
&= f(\bw^*) + \frac{s+1}{2s-1} (f(\bz_s) - f(\bw^*)) + \\
&\quad \frac{s-2}{2s-1} (f(\by_{s-2}) - f(\bw^*))  - \frac{L(s-2)}{2(s+1)} \|\by_{s-1} - \by_{s-2} \|^2 \\
&\geq f(\bw^*) + \frac{1}{2} (f(\bz_s) - f(\bw^*)) - \frac{L}{2} \|\by_{s-1} - \by_{s-2} \|^2,
\end{align*}
where the equality is by adding and subtracting $f(\bw^*)$ and the fact $\by_{s-1} - \by_{s-2} = \frac{s+1}{2s-1}(\bz_{s} - \by_{s-2})$,
and the last inequality is by $f(\by_{s-2}) \geq f(\bw^*)$ and trivial relaxations.

Rearranging gives $f(\bz_s) - f(\bw^*) \leq 2(f(\by_{s-1} - f(\bw^*)) +  L \|\by_{s-1} - \by_{s-2} \|^2$.
Applying Cauchy-Schwarz inequality, strong convexity and the fact $\mu \geq 1$, we continue with
\begin{align*}
&\quad f(\bz_s) - f(\bw^*) \\
&\leq 2 (f(\by_{s-1} - f(\bw^*)) \\
&\quad\quad\quad + 2L (\|\by_{s-1} - \bw^* \|^2  + \|\by_{s-2} - \bw^* \|^2 ) \\
&\leq 2 (f(\by_{s-1} - f(\bw^*))  \\
&\quad\quad\quad + 4\mu ( f(\by_{s-1}) - f(\bw^*)  + f(\by_{s-2}) - f(\bw^*) ) \\ 
&\leq 6\mu (f(\by_{s-1} - f(\bw^*))  + 4\mu (f(\by_{s-2}) - f(\bw^*) ),
\end{align*}
For $s \geq 3$, we use the inductive assumption to show
$\E[f(\bz_s) - f(\bw^*) ] \leq \frac{48\mu LD_t^2}{(s-1)s} + \frac{32\mu LD_t^2}{(s-2)(s-1)} \leq \frac{448 \mu LD_t^2}{(s+1)^2}$.
The case for $s=2$ can be verified similarly using the bound on $\E[f(\by_{0}) - f(\bw^*)]$ and $\E[f(\by_{1}) - f(\bw^*)]$ (base case).
Finally we bound the term 
$ \E[f(\by_{0}) - f(\bw^*)] \leq \frac{LD^2}{2^t} = \frac{LD_t^2}{2\mu} \leq \frac{32LD_t^2}{(N_t+1)^2} \leq \frac{32LD_t^2}{(s+1)^2}$,
and conclude that the variance $\E[\| \tilde\nabla_s - \nabla f(\bz_s)\|^2$ is at most
$\frac{6L}{m_{t,s}} (\frac{896 \mu LD_t^2}{(s+1)^2} + \frac{32LD_t^2}{(s+1)^2}) \leq \frac{L^2D_t^2}{N_t(s+1)^2}$,
completing the induction by Lemma~\ref{lem:STORC}.
\end{proof}

\section{Experiments}\label{sec:experiments}
To support our theory, we conduct experiments in the multiclass classification problem mentioned in Sec~\ref{subsec:multiclass}.
Three datasets are selected from the LIBSVM repository\footnote{\url{https://www.csie.ntu.edu.tw/~cjlin/libsvmtools/datasets/}} with relatively large number of features, categories and examples, summarized in the Table~\ref{tab:data}.

\begin{table}[t]
\begin{center}
\begin{tabular}{|c|c|c|c|}
\hline
\textbf{dataset} & \#\textbf{features} & \#\textbf{categories} & \#\textbf{examples} \\
\hline
news20 & 62,061  & 20 &	15,935 \\		
\hline
rcv1 & 47,236 & 53 & 15,564	\\		
\hline
aloi & 128 & 1,000 & 108,000 \\
\hline
\end{tabular}
\end{center}
\caption{Summary of datasets}
\label{tab:data} 
\end{table}

\begin{figure*}[t]
\centering
\begin{minipage}{.28\hsize}
\centering
\vskip -.6in
\includegraphics[width=\textwidth]{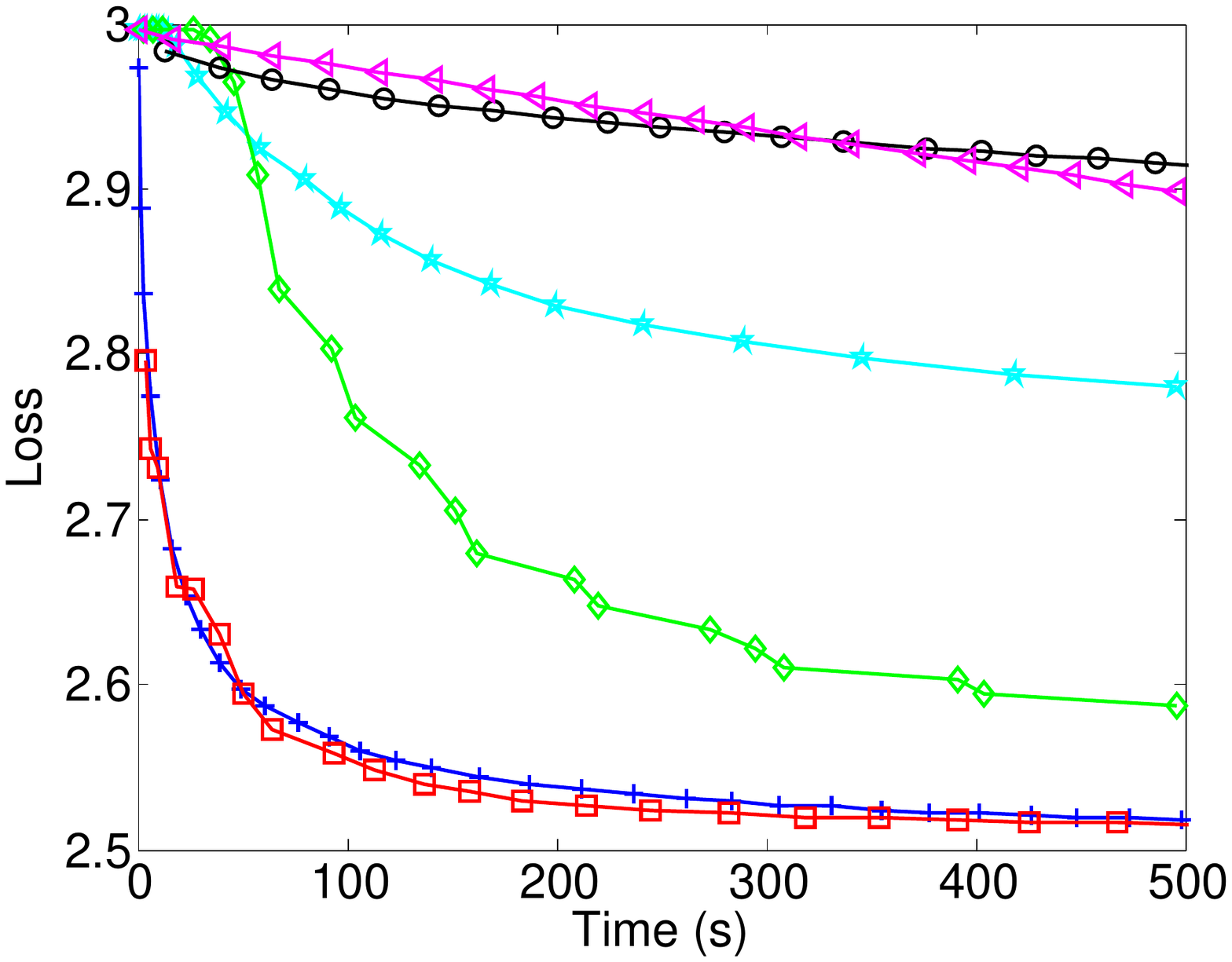}
\vskip -.6in
\captionsetup{labelformat=empty}
\caption{(a) news20}
\end{minipage}
\begin{minipage}{.28\hsize}
\centering
\vskip -.6in
\includegraphics[width=\textwidth]{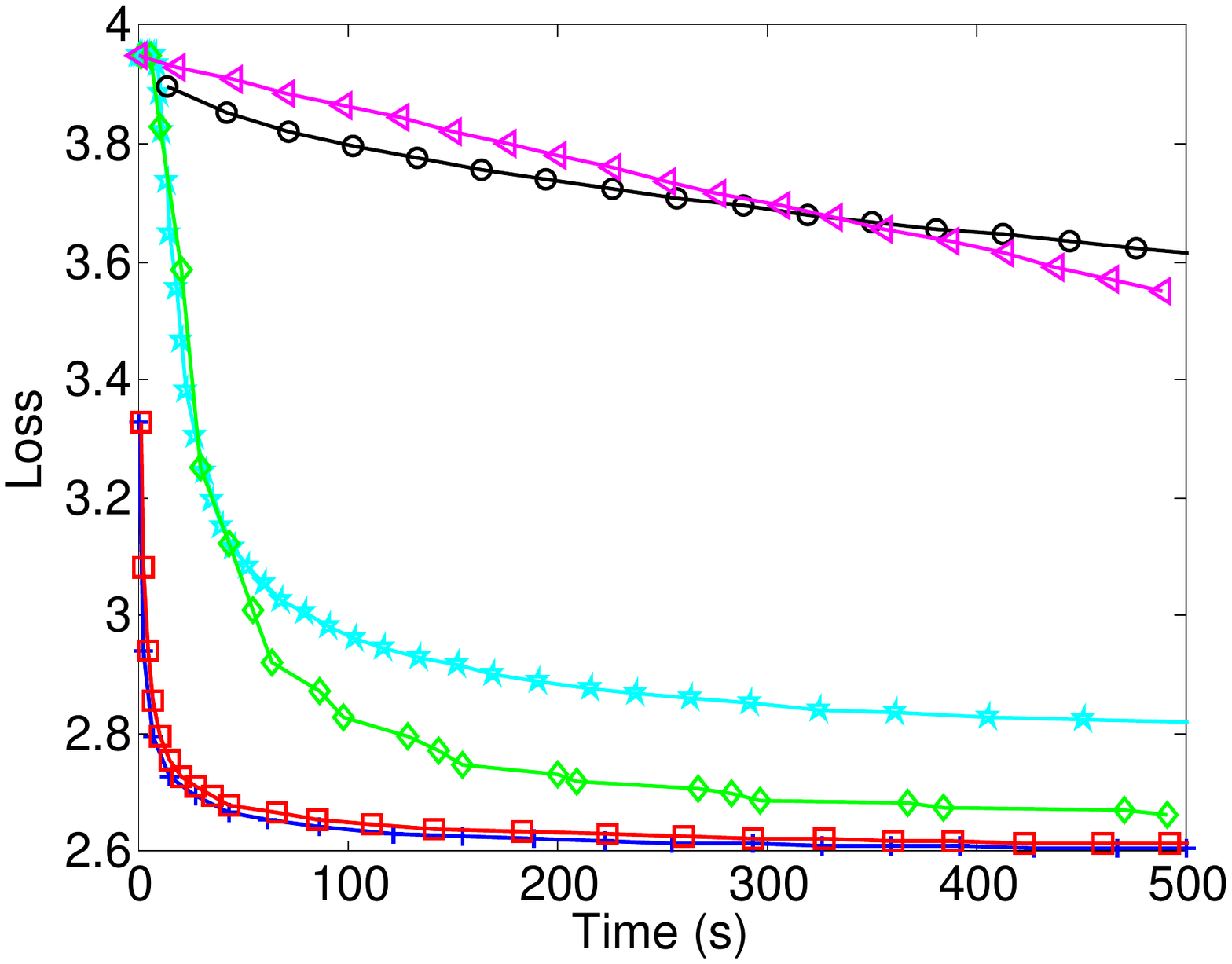}
\vskip -.6in
\captionsetup{labelformat=empty}
\caption{(b) rcv1}
\end{minipage}
\begin{minipage}{.28\hsize}
\centering
\vskip -.6in
\includegraphics[width=\textwidth]{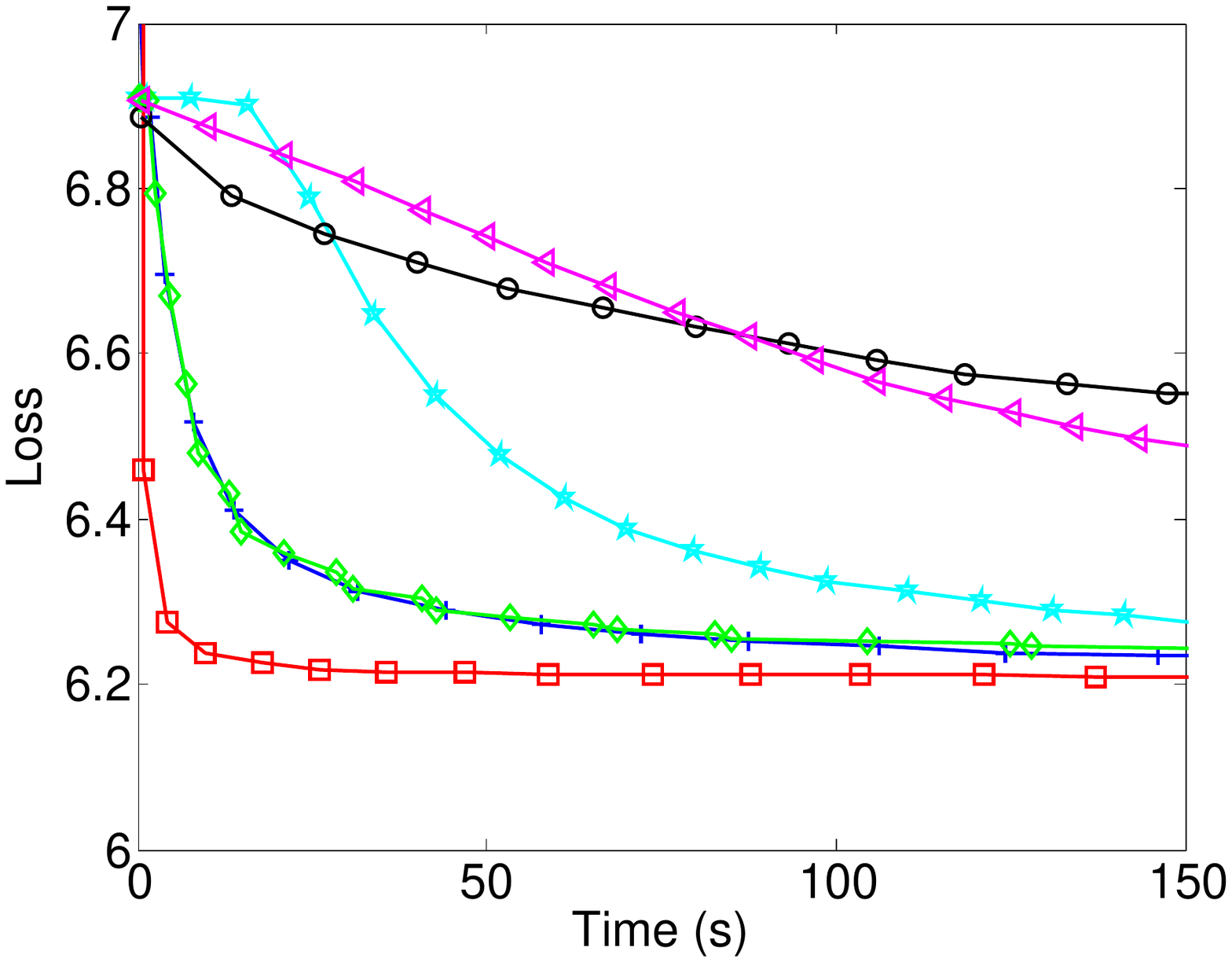}
\vskip -.6in
\captionsetup{labelformat=empty}
\caption{(c) aloi}
\end{minipage}
\begin{minipage}{.12\hsize}
\centering
\vskip -.4in
\includegraphics[width=\textwidth]{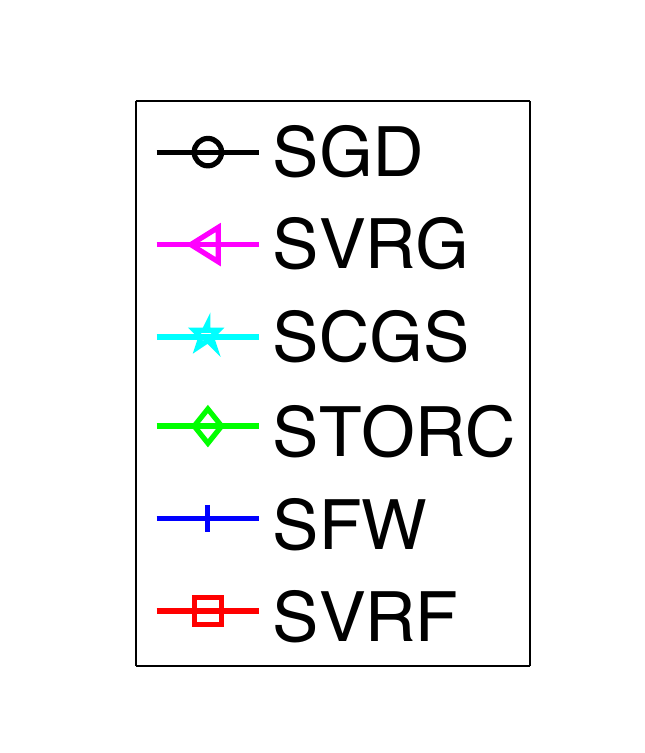}
\end{minipage}
\vspace{10pt}
\setcounter{figure}{0}
\caption{Comparison of six algorithms on three multiclass datasets (best viewed in color)}
\label{fig:experiments}
\end{figure*}

Recall that the loss function is multivariate logistic loss and $\Omega$ is the set of matrices with bounded trace norm $\tau$.
We focus on how fast the loss decreases instead of the final test error rate so that the tuning of $\tau$ is less important, and is fixed to $50$ throughout.

We compare six algorithms. Four of them (SFW, SCGS, SVRF, STORC) are projection-free as discussed, and the other two are standard projected stochastic gradient descent (SGD) and its variance-reduced version (SVRG~\citep{JohnsonZh13}), both of which require expensive projection. 

For most of the parameters in these algorithms, we roughly follow what the theory suggests. 
For example, the size of mini-batch of stochastic gradients at round $k$ is set to $k^2$, $k^3$ and $k$ respectively for SFW, SCGS and SVRF, and is fixed to $100$ for the other three.
The number of iterations between taking two snapshots for variance-reduced methods (SVRG, SVRF and STORC) are fixed to $50$.
The learning rate is set to the typical decaying sequence $c/\sqrt{k}$ for SGD and a constant $c'$ for SVRG as the original work suggests for some best tuned $c$ and $c'$.

Since the complexity of computing gradients, performing linear optimization and projecting are very different, we measure the actual running time of the algorithms and see how fast the loss decreases. 
Results can be found in Figure~\ref{fig:experiments}, where one can clearly observe that for all datasets, SGD and SVRG are significantly slower compared to the others, due to the expensive projection step, highlighting the usefulness of projection-free algorithms.
Moreover, we also observe large improvement gained from the variance reduction technique, especially when comparing SCGS and STORC,
as well as SFW and SVRF on the {\it aloi} dataset.
Interestingly, even though the STORC algorithm gives the best theoretical results, empirically the simpler algorithms SFW and SVRF tend to have consistent better performance.

\section{Omitted Proofs}\label{sec:proofs}

\begin{proof}[\textbf{Proof of Lemma~\ref{lem:variance}}]
Let $\E_i$ denotes the conditional expectation given all the past except the realization of $i$.
We have 
\begin{align*}
&\E_i[ \| \tilde\nabla f(\bw; \bw_0) - \nabla f(\bw) \|^2 ] \\
= &\; \E_i[ \| \nabla f_i(\bw) - \nabla f_i(\bw_0) + \nabla f(\bw_0) - \nabla f(\bw) \|^2 ]\\
= & \;\E_i [\| (\nabla f_i(\bw) - \nabla f_i(\bw^*)) -(  \nabla f_i(\bw_0) - \nabla f_i(\bw^*) ) \\
&+ (\nabla f(\bw_0) - \nabla f(\bw^*)) - ( \nabla f(\bw) - \nabla f(\bw^*)) \|^2 ] \\
\leq\;& 3\E_i [\| \nabla f_i(\bw) - \nabla f_i(\bw^*)\|^2 + \| (\nabla f_i(\bw_0) - \nabla f_i(\bw^*)) \\
&-( \nabla f(\bw_0) - \nabla f(\bw^*)) \|^2
+ \| \nabla f(\bw) - \nabla f(\bw^*) \|^2 ] \\
\leq\;& 3 \E_i [ \| \nabla f_i(\bw) - \nabla f_i(\bw^*)\|^2 + \| \nabla f_i(\bw_0) - \nabla f_i(\bw^*) \|^2 \\
& + \| \nabla f(\bw) - \nabla f(\bw^*) \|^2 ] 
\end{align*}
where the first inequality is Cauchy-Schwarz inequality, and the second one is by the fact $\E_i[\nabla f_i(\bw_0) - \nabla f_i(\bw^*)] = \nabla f(\bw_0) - \nabla f(\bw^*)$ and that the variance of a random variable is bounded by its second moment. 

We now apply Property~\eqref{eq:smoothness1} to bound each of the three terms above. For example,
$\E_i\| \nabla f_i(\bw) - \nabla f_i(\bw^*) \|^2 \leq 2L\E_i[ f_i(\bw) - f_i(\bw^*) - \nabla f_i(\bw^*)^\top (\bw - \bw^*) ] 
=  2L(f(\bw) - f(\bw^*) - \nabla f(\bw^*)^\top (\bw - \bw^*))$,
which is at most $2L(f(\bw) - f(\bw^*))$ by the optimality of $\bw^*$.
Proceeding similarly for the other two terms concludes the proof.
\end{proof}

\begin{proof}[\textbf{Proof of Lemma~\ref{lem:SVRF}}]
For any $s \leq k$, by smoothness we have 
$f(\bx_s) \leq f(\bx_{s-1}) + \nabla f(\bx_{s-1})^\top(\bx_s - \bx_{s-1}) + \frac{L}{2}\|\bx_s - \bx_{s-1}\|^2$.
Plugging in $\bx_s = (1 - \gamma_s) \bx_{s-1} + \gamma_s \bv_s$ gives
$f(\bx_s) \leq f(\bx_{s-1}) + \gamma_s \nabla f(\bx_{s-1})^\top(\bv_s - \bx_{s-1}) + \frac{L\gamma_s^2}{2}\|\bv_s - \bx_{s-1}\|^2$.
Rewriting and using the fact that $\|\bv_s - \bx_{s-1}\|\leq D$ leads to
\begin{align*}
f(\bx_s) &\leq f(\bx_{s-1}) + \gamma_s \tilde\nabla_s^\top(\bv_s - \bx_{s-1}) \\
&+ \gamma_s (\nabla f(\bx_{s-1}) - \tilde\nabla_s)^\top (\bv_s - \bx_{s-1}) + \frac{LD^2\gamma_s^2}{2}.
\end{align*}

The optimality of $\bv_s$ implies $\tilde\nabla_s^\top \bv_s \leq \tilde\nabla_s^\top \bw^*$.
So with further rewriting we arrive at
\begin{align*}
f(\bx_s) &\leq  f(\bx_{s-1}) + \gamma_s \nabla f(\bx_{s-1})^\top(\bw^* - \bx_{s-1}) \\ 
&+ \gamma_s (\nabla f(\bx_{s-1}) - \tilde\nabla_s)^\top (\bv_s - \bw^*) + \frac{LD^2\gamma_s^2}{2}.
\end{align*}

By convexity, term $\nabla f(\bx_{s-1})^\top(\bw^* - \bx_{s-1})$ is bounded by $f(\bw^*) - f(\bx_{s-1})$,
and by Cauchy-Schwarz inequality, term $(\nabla f(\bx_{s-1}) - \tilde\nabla_s)^\top (\bv_s - \bw^*)$ is bounded by $D \| \tilde\nabla_s - \nabla f(\bx_{s-1}) \|$, which in expectation is at most $\frac{LD^2}{s+1}$ by the condition on $\E[\| \tilde\nabla_s - \nabla f(\bx_{s-1})\|^2 ]$ and Jensen's inequality. Therefore we can bound $\E[f(\bx_s) - f(\bw^*)]$ by
\begin{align*}
& (1-\gamma_s) \E[f(\bx_{s-1}) - f(\bw^*)] + \frac{LD^2\gamma_s}{s+1} + \frac{LD^2\gamma_s^2}{2} \\
=\;& (1-\gamma_s) \E[f(\bx_{s-1}) - f(\bw^*)] + LD^2\gamma_s^2.
\end{align*}
Finally we prove $\E[f(\bx_k) - f(\bw^*)] \leq \frac{4LD^2}{k+2}$ by induction.
The base case is trival: $\E[f(\bx_1) - f(\bw^*)]$ is bounded by $(1-\gamma_1) \E[f(\bx_{0}) - f(\bw^*)] + LD^2\gamma_1^2 = LD^2$ since $\gamma_1 = 1$.
Suppose $\E[f(\bx_{s-1}) - f(\bw^*)]  \leq \frac{4LD^2}{s+1}$ then with $\gamma_s = \frac{2}{s+1}$ we bound $\E[f(\bx_s) - f(\bw^*)]$ by 
\[ \frac{4LD^2}{s+1}\left(1 - \frac{2}{s+1} + \frac{1}{s+1} \right) \leq \frac{4LD^2}{s+2}, \]
completing the induction.
\end{proof}

\section{Conclusion and Open Problems}
We conclude that the variance reduction technique, previously shown to be highly useful for gradient descent variants, can also be very helpful in speeding up projection-free algorithms.
The main open question is, in the strongly convex case, whether the number of stochastic gradients for STORC can be improved from $\scO(\mu^2\ln \frac{1}{\epsilon})$ to $\scO(\mu\ln \frac{1}{\epsilon})$, which is typical for gradient descent methods, and whether the number of linear optimizations can be improved from $\scO(\frac{1}{\epsilon})$ to $\scO(\ln \frac{1}{\epsilon})$.

\paragraph{Acknowledgements}
The authors acknowledge support from the National Science Foundation grant IIS-1523815 and a Google research award.

\bibliography{ref}
\bibliographystyle{icml2016}

\newpage
\onecolumn
\appendix

\iflong
\begin{center}
\bf \Large Supplementary material for \\ ``Variance-Reduced and Projection-Free Stochastic Optimization''
\end{center}
\fi

\section{Proof of Property~\eqref{eq:smoothness1}}\label{app:smoothness}

\begin{proof}
We drop the subscript $i$ for conciseness. 
Define $g(\bw) = f(\bw) - \nabla f(\bv)^\top \bw$,
which is clearly also convex and $L$-smooth on $\Omega$.
Since $\nabla g(\bv) = \bzero$, $\bv$ is one of the minimizers of $g(\bw)$.
Therefore we have
\begin{align*}
g(\bv) - g(\bw) &\leq  g(\bw - \frac{1}{L}\nabla g(\bw)) - g(\bw) \\
&\leq \nabla g(\bw)^\top (\bw - \frac{1}{L}\nabla g(\bw) - \bw) + \frac{L}{2} \| \bw - \frac{1}{L}\nabla g(\bw) - \bw \|^2  \tag{by smoothness of $g$}\\
&= -\frac{1}{2L}\| \nabla g(\bw) \|^2 = -\frac{1}{2L}\| \nabla f(\bw) -  \nabla f(\bv)\|^2
\end{align*}
Rearranging and plugging in the definition of $g$ concludes the proof.
\end{proof}

\section{Analysis for SFW}\label{app:SFW}
The concrete update of SFW is 
\begin{equation*}
\begin{split}
\bv_k &= \argmin_{\bv \in \Omega} \tilde\nabla_k^\top \bv \\
\bw_k &= (1 - \gamma_k) \bw_{k-1} + \gamma_k \bv_k
\end{split}
\end{equation*}
where $\tilde\nabla_k$ is the average of $m_k$ iid samples of stochastic gradient $\nabla f_i(\bw_{k-1})$.
The convergence rate of SFW is presented below.

\begin{theorem}
If each $f_i$ is $G$-Lipschitz, then with $ \gamma_k = \frac{2}{k+1}$ and $m_k = \left( \frac{G(k+1)}{LD}\right)^2$, SFW ensures for any $k$,
\[ \E[f(\bw_k) - f(\bw^*)] \leq \frac{4LD^2}{k+2}. \] 
\end{theorem}

\begin{proof}
Similar to the proof of Lemma~\ref{lem:SVRF}, we first proceed as follows, 
\begin{align*}
f(\bw_k) &\leq f(\bw_{k-1}) + \nabla f(\bw_{k-1})^\top(\bw_k - \bw_{k-1}) + \frac{L}{2}\|\bw_k - \bw_{k-1}\|^2 \tag{smoothness} \\
&= f(\bw_{k-1}) + \gamma_k \nabla f(\bw_{k-1})^\top(\bv_k - \bw_{k-1}) + \frac{L\gamma_k^2}{2}\|\bv_k - \bx_{k-1}\|^2  \tag{$\bw_k - \bw_{k-1}  = \gamma_k (\bv_k - \bw_{k-1})$} \\ 
&\leq f(\bw_{k-1}) + \gamma_k \tilde\nabla_k^\top(\bv_k - \bw_{k-1}) + \gamma_k (\nabla f(\bw_{k-1}) - \tilde\nabla_k)^\top (\bv_k - \bw_{k-1}) + \frac{LD^2\gamma_k^2}{2}  \tag{$\|\bv_k - \bw_{k-1}\| \leq D$} \\
&\leq f(\bw_{k-1}) + \gamma_k \tilde\nabla_k^\top(\bw^* - \bw_{k-1}) + \gamma_k (\nabla f(\bw_{k-1}) - \tilde\nabla_k)^\top (\bv_k - \bw_{k-1}) + \frac{LD^2\gamma_k^2}{2} \tag{by optimality of $\bv_k$} \\
&= f(\bw_{k-1}) + \gamma_k \nabla f(\bw_{k-1})^\top(\bw^* - \bw_{k-1}) + \gamma_k (\nabla f(\bw_{k-1}) - \tilde\nabla_k)^\top (\bv_k - \bw^*) + \frac{LD^2\gamma_k^2}{2} \\
&\leq f(\bw_{k-1}) + \gamma_k (f(\bw^*) - f(\bw_{k-1})) + \gamma_k D \| \tilde\nabla_k - \nabla f(\bw_{k-1}) \| + \frac{LD^2\gamma_k^2}{2}, 
\end{align*}
where the last step is by convexity and Cauchy-Schwarz inequality. 
Since $f_i$ is $G$-Lipschitz, with Jensen's inequality, we further have $\E[ \|\tilde\nabla_k  - \nabla f(\bw_{k-1}) \|] \leq \sqrt{\E[ \|\tilde\nabla_k  - \nabla f(\bw_{k-1}) \|^2] } \leq\frac{G}{ \sqrt{m_k}} $, 
which is at most $\frac{LD\gamma_k}{2}$ with the choice of $\gamma_k$ and $m_k$. 
So we arrive at
$\E[f(\bw_k) - f(\bw^*)] \leq (1-\gamma_k) \E[f(\bw_{k-1}) - f(\bw^*)] + LD^2\gamma_k^2$.
It remains to use a simple induction to conclude the proof.
\end{proof}

Now it is clear that to achieve $1 - \epsilon$ accuracy, SFW needs $\scO(\frac{LD^2}{\epsilon})$ iterations, 
and in total $\scO(\frac{G^2}{L^2D^2} (\frac{LD^2}{\epsilon})^3) = \scO(\frac{G^2LD^4}{\epsilon^3})$ stochastic gradients.

\section{Proof of Lemma~\ref{lem:STORC}}\label{app:STORC}
\begin{proof}
Let $\bdelta_s = \tilde\nabla_s - \nabla f(\bz_s)$. For any $s \leq k$, we proceed as follows:
\begin{align*}
f(\by_s) &\leq f(\bz_s) + \nabla f(\bz_s)^\top (\by_s - \bz_s) + \frac{L}{2} \norm{\by_s - \bz_s}^2  \tag{by smoothness} \\
&= (1 - \gamma_s) ( f(\bz_s) + \nabla f(\bz_s)^\top (\by_{s-1} - \bz_s)) +
    \gamma_s (f(\bz_s) + \nabla f(\bz_s)^\top (\bw^* - \bz_s) ) + \gamma_s \nabla f(\bz_s)^\top (\bx_s - \bw^*)  \notag \\
    &\quad\quad + \frac{L\gamma_s^2}{2} \norm{\bx_s - \bx_{s-1}}^2   \tag{by definition of $\by_s$ and $\bz_s$} \\
&\leq  (1 - \gamma_s) f(\by_{s-1})  +  \gamma_s f(\bw^*)  + \gamma_s \nabla f(\bz_s)^\top (\bx_s - \bw^*) + \frac{L\gamma_s^2}{2} \norm{\bx_s - \bx_{s-1}}^2  \tag{by convexity} \\    
&= (1 - \gamma_s) f(\by_{s-1})  +  \gamma_s f(\bw^*)  + \gamma_s \tilde\nabla_s^\top (\bx_s - \bw^*) + \frac{L\gamma_s^2}{2} \norm{\bx_s - \bx_{s-1}}^2 + \gamma_s \bdelta_s^\top (\bw^* - \bx_s)  \notag \\    
&\leq (1 - \gamma_s) f(\by_{s-1})  +  \gamma_s f(\bw^*)  + \gamma_s \eta_{t,s} - \gamma_s \beta_s (\bx_s - \bx_{s-1})^\top (\bx_s - \bw^*)  + \frac{L\gamma_s^2}{2} \norm{\bx_s - \bx_{s-1}}^2 + \gamma_s \bdelta_s^\top (\bw^* - \bx_s)  \tag{by Eq.~\eqref{eq:gap}} \\    
&= (1 - \gamma_s) f(\by_{s-1})  +  \gamma_s f(\bw^*)  + \gamma_s \eta_{t,s} + \frac{\beta_s\gamma_s}{2} (\norm{\bx_{s-1} - \bw^*}^2 -  \norm{\bx_s - \bw^*}^2)  +  \notag \\
&\quad\quad \frac{\gamma_s}{2} \left((L\gamma_s - \beta_s) \norm{\bx_s - \bx_{s-1}}^2 + 2\bdelta_s^\top (\bx_{s-1} - \bx_s) + 
  2\bdelta_s^\top (\bw^* - \bx_{s-1})  \right) \notag \\    
&\leq (1 - \gamma_s) f(\by_{s-1})  +  \gamma_s f(\bw^*)  + \gamma_s \eta_{t,s} + \frac{\beta_s\gamma_s}{2} (\norm{\bx_{s-1} - \bw^*}^2 -  \norm{\bx_s - \bw^*}^2)  +  \frac{\gamma_s}{2} \left( \frac{\norm{\bdelta_s}^2}{\beta_s - L\gamma_s} + 2\bdelta_s^\top (\bw^* - \bx_{s-1})  \right),
\end{align*}
where the last inequality is by the fact $\beta_s \geq L\gamma_s$ and thus
\[ (L\gamma_s - \beta_s) \norm{\bx_s - \bx_{s-1}}^2 + 2\bdelta_s^\top (\bx_{s-1} - \bx_s) 
=   \frac{\norm{\bdelta_s}^2}{\beta_s - L\gamma_s} - (\beta_s - L\gamma_s) \norm{\bx_s - \bx_{s-1} - \frac{\bdelta_s}{\beta_s - L\gamma_s} }^2 
\leq  \frac{\norm{\bdelta_s}^2}{\beta_s - L\gamma_s}. \]
Note that $\E[\bdelta_s^\top (\bw^* - \bx_{s-1}) ] = \bzero$. 
So with the condition $\E[\norm{\bdelta_s}^2] \leq \frac{L^2 D_t^2}{N_t(s+1)^2} \defeq \sigma_s^2$ we arrive at
\[ \E[f(\by_s) - f(\bw^*)] \leq  (1 - \gamma_s) \E[f(\by_{s-1}) - f(\bw^*)] + \gamma_s \left(\eta_{t,s} + \frac{\beta_s}{2} (\E[\norm{\bx_{s-1} - \bw^*}^2] -  \E[\norm{\bx_s - \bw^*}^2])  
+ \frac{\sigma_s^2}{2(\beta_s - L\gamma_s)} \right). \]

Now define $\Gamma_s = \Gamma_{s-1} (1 - \gamma_s) $ when $s > 1$ and $\Gamma_1 = 1$. 
By induction, one can verify $\Gamma_s = \frac{2}{s(s+1)}$ and the following:
\[ \E[f(\by_k) - f(\bw^*)] \leq \Gamma_k \sum_{s=1}^k \frac{\gamma_s}{\Gamma_s} \left(\eta_{t,s} + \frac{\beta_s}{2} (\E[\norm{\bx_{s-1} - \bw^*}^2] -  \E[\norm{\bx_s - \bw^*}^2])  
+ \frac{\sigma_s^2}{2(\beta_s - L\gamma_s)} \right), \]
which is at most
\[
\Gamma_k \sum_{s=1}^k \frac{\gamma_s}{\Gamma_s} \left( \eta_ s + \frac{\sigma_s^2}{2(\beta_s - L\gamma_s)} \right)  +
\frac{\Gamma_k}{2} \left( \frac{\gamma_1\beta_1}{\Gamma_1} \E[\norm{\bx_{0} - \bw^*}^2]  + \sum_{s=2}^k \left(\frac{\gamma_s\beta_s}{\Gamma_s} - \frac{\gamma_{s-1}\beta_{s-1}}{\Gamma_{s-1}}\right)\E[\norm{\bx_{s-1} - \bw^*}^2] \right) . 
\]
Finally plugging in the parameters $\gamma_s$, $\beta_s$, $\eta_{t,s}$, $\Gamma_s$ and the bound $\E[\|\bx_0 - \bw^*\|^2] \leq D_t^2$ concludes the proof:
\[ \E[f(\by_k) - f(\bw^*)] \leq \frac{2}{k(k+1)} \sum_{s=1}^k k \left( \frac{2LD_t^2}{N_t k} + \frac{L D_t^2}{2N_t(k+1)} \right)  +
\frac{3L D_t^2}{k(k+1)} \leq \frac{8LD_t^2}{k(k+1)}.  \]
\end{proof}

\end{document}